\documentclass[12pt, a4paper]{article}

\usepackage[a4paper, margin=2.5cm]{geometry}
\usepackage{amsmath, amssymb, amsthm}
\usepackage{booktabs}
\usepackage{graphicx}
\usepackage{xcolor}
\usepackage{hyperref}
\usepackage{microtype}
\usepackage{parskip}
\usepackage{setspace}
\usepackage{enumitem}
\usepackage{natbib}
\usepackage{mathtools}

\hypersetup{colorlinks=true, linkcolor=blue!60!black,
            citecolor=blue!60!black, urlcolor=blue!60!black}

\newtheorem{definition}{Definition}
\newtheorem{proposition}{Proposition}
\newtheorem{remark}{Remark}

\newcommand{\smece}{\textsc{smece}}
\newcommand{\ece}{\textsc{ece}}
\newcommand{\Sb}{S_b}
\newcommand{\phatb}{\hat{p}_b}
\newcommand{\ybarb}{\bar{y}_b}
\newcommand{\yhardb}{\hat{y}^{\mathrm{hard}}_b}
\DeclareMathOperator{\sigmoid}{\sigma}
\newcommand{\given}{\mid}

\begin{document}

\begin{center}
  {\LARGE\bfseries Soft Mean Expected Calibration Error (SMECE):\\[4pt]
   A Calibration Metric for Probabilistic Labels}\\[18pt]
  {\large Michael Leznik}\\[6pt]
  {\normalsize March 2026}
\end{center}

\begin{abstract}
The Expected Calibration Error (\ece{}), the dominant calibration metric in
machine learning, compares predicted probabilities against empirical frequencies
of binary outcomes. This is appropriate when labels are binary events. However,
many modern settings produce labels that are themselves probabilities rather
than binary outcomes: a radiologist's stated confidence, a teacher model's soft
output in knowledge distillation, a class posterior derived from a generative
model, or an annotator agreement fraction. In these settings, \ece{} commits
a category error --- it discards the probabilistic information in the label by
forcing it into a binary comparison. The result is not a noisy approximation
that more data will correct. It is a structural misalignment that persists and
converges to the wrong answer with increasing precision as sample size grows.

We introduce the \textbf{Soft Mean Expected Calibration Error (\smece{})},
a calibration metric for settings where labels are of probabilistic nature.
The modification to the \ece{} formula is one line: replace the empirical
hard-label fraction in each prediction bin with the mean probability label of
the samples in that bin. \smece{} reduces exactly to \ece{} when labels are
binary, making it a strict generalisation.

The theoretical grounding draws on the Bayesian framework, which provides the
most principled account of why probabilistic labels arise and why evaluating
them against binary realisations is wrong. Under a Gaussian class-conditional
generative model with equal priors, the posterior $P(Y=1 \given x)$ is
analytically $\sigmoid(kx)$ --- derived, not declared. The binary outcome is
one realisation drawn from this posterior; the posterior is the complete answer.
However, the case for \smece{} is not a case for Bayesian methodology. It
holds equally for frequentist mixture models, knowledge distillation, and any
framework that produces continuous probability labels. The unifying condition
is simpler: whenever the calibration target is itself a probability, \smece{}
is the correct metric and \ece{} is not.

Four experiments demonstrate the consequences. A model that correctly learns
the posterior achieves \smece{} = 0.0000 exactly at every sample size, while
its \ece{} converges to 0.1151 --- a stable non-zero quantity that grows more
precisely wrong as $n$ increases. At low signal-to-noise, \ece{}'s ranking
accuracy falls to 40\%, below chance, while \smece{} achieves 90\%.

\medskip
\noindent\textbf{Keywords:} calibration, probabilistic labels, Expected
Calibration Error, soft labels, Bayesian posterior, knowledge distillation,
uncertainty quantification.
\end{abstract}

\newpage

\section{Introduction}
\label{sec:intro}

\subsection{A Motivating Example}

Consider a radiologist examining a scan. Based on the image, prior clinical
knowledge, and available patient history, she states: ``I am 70\% confident
this is malignant.'' This probability is not a frequency. There is no
sequence of identical scans being run repeatedly. It is the radiologist's
rational summary of her uncertainty given the evidence available at that
moment --- a degree of belief.

Now consider a model trained to assist with this task. The model takes the
scan as input $x$ and outputs a predicted probability $\hat{p}(x)$. The
calibration question is: does the model correctly represent the uncertainty
that a rational agent should have given this scan? In our mapping:

\begin{itemize}
  \item $x$ is the scan --- the observed evidence.
  \item $P(Y=1 \given x)$ is the true posterior --- what a rational agent
    \textit{should} believe given this scan, derived from all prior knowledge
    about how scans relate to diagnoses.
  \item $\hat{p}(x)$ is the model's prediction --- what the model outputs
    given the scan.
  \item $y \in \{0,1\}$ is the binary outcome --- what the biopsy
    eventually reveals.
\end{itemize}

A model that outputs $\hat{p}(x) = 0.70$ for a scan whose true posterior is
$P(Y=1 \given x) = 0.70$ is perfectly calibrated. It correctly represents
the uncertainty given the available evidence. The biopsy may return 0 or 1,
but that does not make 0.70 wrong: it was the right answer given what was
knowable at prediction time. Of course the binary outcome matters enormously
--- it determines the patient's treatment. But it answers a different question:
not ``was the model's uncertainty correct?'' but ``what did the biopsy find?''

\ece{} conflates these two questions. It evaluates whether $\hat{p}(x)$
matches the biopsy result $y$, not whether it matches the posterior
$P(Y=1 \given x)$. A model that outputs 0.70 will be penalised whenever
the biopsy returns 0, even though 0.70 was the correct degree of belief given
the scan. A model that overconfidently outputs 0.95 will be penalised less,
because overconfidence aligns better with the binary outcome. \smece{} asks
the right question: does $\hat{p}(x)$ match $P(Y=1 \given x)$?

\subsection{Binary Labels and Probabilistic Labels}

Machine learning labels come in two fundamentally different forms.

A \textit{binary label} $y \in \{0, 1\}$ records an observed outcome:
the patient has cancer or does not; the transaction is fraudulent or is not.
Binary labels are the native input of \ece{}: the metric was designed to
evaluate whether predicted probabilities match the frequency of such outcomes.
Under the frequentist interpretation of probability --- where $P(Y=1) = 0.7$
means that 70\% of trials in a long-run reference class are positive ---
binary labels and \ece{} are perfectly aligned.

A \textit{probabilistic label} $y \in [0,1]$ is itself a probability. It does
not record a binary event; it records a degree of certainty about an event.
The radiologist's 70\% confidence, the teacher model's soft output of 0.73,
the posterior $P(Y=1 \given x)$ derived from a generative model --- all are
probabilistic labels. Under the Bayesian interpretation of probability,
associated with Thomas Bayes and formalised by Bruno de Finetti, such labels
are degrees of belief: rational summaries of what is known given the available
evidence. There is no binary truth underneath them waiting to be revealed.
The label \textit{is} the answer.

\ece{} was not designed for probabilistic labels. Applying it to them forces
a category error: the metric treats the probabilistic label as if it were a
noisy binary outcome, discards its distributional content, and evaluates the
model against a binary approximation of the label that was never the ground truth.

\subsection{ECE is a Binary Metric}

The Expected Calibration Error \citep{naeini2015, guo2017} asks: over all
inputs to which the model assigns probability $\hat{p} \approx p$, does the
fraction of positive outcomes equal $p$? Formally, it computes the weighted
mean absolute difference between the mean prediction $\phatb$ and the
empirical positive fraction $\yhardb$ within equal-width bins on $[0,1]$:
\begin{equation}
  \ece{} = \sum_{b=1}^{B} \frac{|\Sb|}{n}
    \bigl|\, \phatb - \yhardb \,\bigr|.
  \label{eq:ece}
\end{equation}
\ece{} is a binary metric. Its calibration target $\yhardb$ is the
empirical fraction of binary outcomes within each prediction bin --- a
quantity that is only defined when labels are binary. At its core, \ece{}
compares:
\begin{center}
  \textit{probabilistic prediction} $\longleftrightarrow$
  \textit{binary outcome}
\end{center}
This comparison is well-posed when the label is binary. It is a category
error when the label is itself a probability, because the metric discards
the distributional content of the label and replaces it with a single
binary value.

The reason \ece{} is binary is that it is frequentist in nature: it
interprets calibration as agreement between predicted probabilities and
long-run outcome frequencies, which are inherently binary. This is the
right criterion when labels are binary outcomes. When labels are
probabilistic --- when the label is a degree of belief, a posterior,
or a soft target --- \ece{} is evaluating the model against a
binary approximation of the label, not against the label itself.
A single binary outcome cannot validate or invalidate a probabilistic
label, for the same reason that a single coin flip cannot validate or
invalidate a 70\% probability statement.

\subsection{The Core Problem}

The conflict between \ece{} and probabilistic labels can be stated precisely.
Let $p^*(x) \in [0,1]$ be the true probabilistic label for input $x$. A model
that correctly learns the label outputs $\hat{p}(x) = p^*(x)$ for all $x$.
This model is perfectly calibrated in the probabilistic label sense: its
predicted probability equals the true probability at every input.

Now evaluate this model with \ece{} against binary outcomes $y \sim
\mathrm{Bernoulli}(p^*(x))$. In each prediction bin $b$, the mean prediction
is $\phatb = \bar{p}^*(x)_b$ --- the mean true posterior in that bin. The
empirical fraction $\yhardb$ is an average of Bernoulli realisations, which
converges to the same mean posterior as $n \to \infty$. So at infinite $n$,
\ece{} $\to 0$ for the perfectly calibrated model --- but only in the limit.
At finite $n$, Bernoulli noise drives \ece{} above zero even for the perfect
model. More importantly, when the posterior $p^*(x)$ is not near 0 or 1 ---
when there is genuine uncertainty --- the Bernoulli realisations are noisy
enough that \ece{} systematically misranks models.

The situation is worse than finite-sample noise. Consider a discrete
threshold: $y_{\mathrm{hard}} = \mathbf{1}[x \geq 0]$. This is what
\ece{} uses if hard labels are obtained by thresholding. Now the calibration
target is a step function, and the perfect posterior model
$\hat{p}(x) = \sigmoid(kx)$ will \textit{never} achieve \ece{} $= 0$ at
any sample size. Experiment~4 demonstrates this: ECE(A) $= 0.1151$ at
$n = 500$, $n = 1000$, $n = 5000$, and $n = 10\,000$, with declining
variance throughout. \ece{} converges to the wrong answer with increasing
precision.

\subsection{SMECE as a Calibration Metric for Probabilistic Labels}

We introduce the Soft Mean Expected Calibration Error (\smece{}), which
evaluates model predictions against the posterior probability rather than
against binary realisations. The modification to equation~\eqref{eq:ece}
is one line: replace $\yhardb$ with the mean posterior probability
$\ybarb$ of the samples in bin $b$. A model that correctly learns the
posterior achieves \smece{} $= 0$ exactly and identically --- not as a
limit but as an algebraic fact.

\smece{} is not merely a technical adjustment. It reflects a different
epistemological commitment. \ece{} asks: do your probabilities match
outcome frequencies? \smece{} asks: do your probabilities match the
probabilistic label? These are different questions. Whenever the label is
itself a probability, only the second question is appropriate --- regardless
of whether the label arose from a Bayesian posterior, a frequentist mixture,
or a teacher network.

\subsection{Relationship to Existing Work}

The problem of evaluating calibration when labels are not binary has been
noted in recent literature, primarily from the direction of annotator
disagreement. \citet{baan2022} argue that \ece{} is invalid when humans
inherently disagree on labelling and propose distribution-level calibration
metrics (DistCE, EntCE) for multiclass NLP settings. \citet{khurana2024}
extend this to selective prediction under crowd disagreement. These
contributions share our diagnosis --- that binary outcomes are the wrong
calibration target in some settings --- but motivate soft labels as
aggregations of annotator votes rather than as Bayesian posteriors, and
operate in the multiclass NLP setting with full per-instance label
distributions.

Our contribution is distinct in three ways. First, our central claim is
that \ece{} is the wrong metric whenever labels are probabilistic, not
merely when annotators disagree; the Bayesian framework provides the most
principled grounding for this claim, with the posterior derived from a
generative model rather than aggregated from votes. Second, we operate
in the binary classification setting and require only a single soft label
per instance, making \smece{} applicable in any setting that produces
continuous probability labels. Third, we identify the unifying condition
--- labels of probabilistic nature --- that subsumes Bayesian posteriors,
frequentist mixtures, distillation, and expert elicitation under a single
metric, a generalisation absent from the annotator disagreement literature.

\subsection{Contributions and Outline}

This paper makes three contributions. First, we identify probabilistic
labels as the condition under which \ece{} is structurally wrong, and
establish this analytically using the Bayesian framework as the most
principled grounding. Second, we define \smece{} as the calibration
metric for probabilistic labels and show it is a strict generalisation
of \ece{}, reducing to it exactly when labels are binary. Third, we
present four controlled experiments using a generative model whose posterior
is analytically $\sigmoid(kx)$, demonstrating that \smece{} correctly
identifies the probability-matching model at every sample size while
\ece{} systematically fails.

Section~\ref{sec:background} reviews \ece{} and related work.
Section~\ref{sec:epistemology} analyses why \ece{} fails structurally
for probabilistic labels. Section~\ref{sec:smece} defines \smece{}. Section~\ref{sec:simulation}
describes the generative model. Sections~\ref{sec:exp1}--\ref{sec:exp4}
present the experiments. Section~\ref{sec:discussion} discusses implications.
Section~\ref{sec:conclusion} concludes.

\section{Background}
\label{sec:background}

\subsection{Expected Calibration Error}

\ece{} partitions the unit interval into $B$ equal-width prediction bins,
computes the mean prediction $\phatb$ and empirical positive fraction
$\yhardb$ within each bin, and returns the sample-weighted mean absolute
difference. Known limitations include sensitivity to $B$, noise in sparse
bins, and concentration effects under unequal class balance \citep{kumar2019}.
Adaptive equal-mass binning, isotonic regression, and kernel-based calibration
estimators have been proposed as alternatives \citep{zhang2020, gruber2022}.
All of these are frequentist: they evaluate against binary outcomes.

\subsection{Calibration Beyond Binary Outcomes}

Temperature scaling \citep{guo2017} and Platt scaling learn a post-hoc
transformation of model logits to improve \ece{}. They are evaluated
against binary outcomes and are frequentist methods.

Proper scoring rules \citep{gneiting2007, degroot1983} provide a
decision-theoretic framework. The Brier score
$\mathrm{BS} = n^{-1}\sum_i(\hat{p}_i - y_i)^2$ extends naturally to
$y_i \in [0,1]$ and is minimised at $\hat{p}_i = y_i$. The Brier score
is a proper scoring rule and is therefore compatible with the Bayesian
setting when $y_i$ is the posterior probability. However, it provides
a single aggregate and does not offer the bin-level diagnostic detail
of the reliability diagram. \smece{} provides this detail.

\subsection{Soft Labels and Annotator Disagreement}

\citet{baan2022} demonstrate that \ece{} is theoretically invalid under
inherent human disagreement and propose DistCE, EntCE, and RankCS as
alternatives for multiclass NLP. \citet{khurana2024} use crowd disagreement
fractions for selective prediction calibration. \citet{dawid1979} provide the foundational multi-annotator model;
\citet{uma2021} survey the
broader literature on learning from disagreement. These papers motivate
soft labels via annotator vote aggregation rather than Bayesian posteriors,
require full per-instance label distributions, and address the multiclass
NLP setting. \smece{} addresses the binary setting, requires only a scalar
soft label per instance, and is motivated by a distinct epistemological
argument.

\subsection{Knowledge Distillation and Soft Targets}

\citet{hinton2015} train students on teacher softmax outputs. The student's
calibration against these soft targets is a probabilistic label calibration
question: does the student correctly represent the teacher's output probability?
Evaluating the student with \ece{} against hard test labels asks a different
and less relevant question. \smece{} provides the appropriate metric.

\section{Why ECE Fails for Probabilistic Labels}
\label{sec:epistemology}

\subsection{What ECE Measures}

\ece{} is consistent for the quantity
\begin{equation}
  \mathbb{E}\Bigl[\bigl|\,
    \hat{p}(X) - P(Y=1 \given \hat{p}(X))
  \,\bigr|\Bigr],
  \label{eq:ece_population}
\end{equation}
the expected absolute difference between the model's output and the true
conditional probability of a positive outcome, averaged over the data
distribution \citep{guo2017, kumar2019}. This is a well-defined calibration
target for binary labels. It is the right quantity to estimate when labels
are binary outcomes and the goal is to assess whether the model's stated
probabilities match their empirical frequency. It is a frequentist criterion
by construction: it measures agreement with observed binary frequencies.

\subsection{Why Probabilistic Labels are a Different Target}

Let $p^*(x) \in [0,1]$ be a probabilistic label --- the correct probability
assigned to input $x$, whether it arises from a Bayesian posterior, a
frequentist mixture, or any other process. A perfectly calibrated model
outputs $\hat{p}(x) = p^*(x)$ for all $x$. The \ece{} of this model
against binary outcomes $y_i \in \{0,1\}$ is:
\begin{equation}
  \mathbb{E}\bigl[\,|\hat{p}(X) - P(Y=1 \given \hat{p}(X))|\,\bigr]
  = \mathbb{E}\bigl[\,|p^*(X) - P(Y=1 \given p^*(X))|\,\bigr].
  \label{eq:freq_of_bayes}
\end{equation}
This equals zero only if the probabilistic label exactly matches the binary
outcome frequency at every value of $p^*$ --- that is, only if the label
and the outcome are the same object. When the label is a probability and
the outcome is a binary realisation of it, they are not. The gap between
them is structural: it is the information discarded when a probability is
collapsed to a binary value. This gap does not shrink with sample size;
more data estimates it with increasing precision but does not close it.

\subsection{The Binary Outcome as a Realisation}

In the Bayesian generative model, the binary outcome $y \sim
\mathrm{Bernoulli}(p^*(x))$ is a realisation drawn from the posterior.
The posterior itself, $p^*(x)$, is the complete statement of what is known.
Using $y$ as the calibration target discards the distributional information
in $p^*(x)$: it replaces a probability with one draw from that probability.

To make this concrete: if $p^*(x) = 0.51$ for some input $x$, then $y$ will
be 0 or 1, each with approximately equal probability. A model that outputs
$0.51$ is perfectly calibrated in the probabilistic label sense. \ece{} will penalise
it by approximately $|0.51 - y|$, which is close to $0.51$ on average ---
a large calibration error for what is in fact a perfectly calibrated model.
A model that outputs $0.99$ will be penalised only when $y = 0$, which
happens with probability $0.49$. In expectation, \ece{} prefers the
overconfident model.

This is the failure mode. It is not finite-sample noise. It is a systematic
inversion of the calibration ranking caused by evaluating against a binary
realisation rather than the posterior.

\section{Soft Mean Expected Calibration Error}
\label{sec:smece}

\subsection{Definition}

Let $\{(\hat{p}_i, p^*_i)\}_{i=1}^n$ be predicted probabilities and
target posterior probabilities with $\hat{p}_i, p^*_i \in [0,1]$.
Partition $[0,1]$ into $B$ equal-width prediction bins and let
$\Sb = \{i : \hat{p}_i \in [(b-1)/B,\, b/B)\}$. Define:
\begin{equation}
  \phatb = \frac{1}{|\Sb|}\sum_{i\in\Sb}\hat{p}_i,
  \qquad
  \ybarb = \frac{1}{|\Sb|}\sum_{i\in\Sb} p^*_i.
  \label{eq:bin_means}
\end{equation}

\begin{definition}[\smece{}]
\begin{equation}
  \smece{} = \sum_{b=1}^{B} \frac{|\Sb|}{n}
    \bigl|\, \phatb - \ybarb \,\bigr|.
  \label{eq:smece}
\end{equation}
\end{definition}

\smece{} replaces the empirical hard-label fraction $\yhardb$ in \ece{}
with the mean target posterior $\ybarb$. All other structure is identical.

\subsection{Properties}

\begin{proposition}[Exact reduction to \ece{}]
When $p^*_i \in \{0,1\}$ for all $i$, $\ybarb = \yhardb$ identically
and $\smece{} = \ece{}$.
\end{proposition}

\begin{proposition}[Zero for the posterior-matching model]
If $\hat{p}_i = p^*_i$ for all $i$, then $\phatb = \ybarb$ in every
bin and $\smece{} = 0$ exactly. This holds at every sample size, not
as a limit.
\end{proposition}

\begin{proposition}[ECE cannot reach zero when outcomes are thresholded]
\label{prop:ece_floor}
Let $y_i = \mathbf{1}[p^*_i > \tau]$ for some threshold $\tau$. For the
posterior-matching model $\hat{p}_i = p^*_i$, the \ece{} against
$y_i$ satisfies:
\[
  \ece{} \;\xrightarrow{n\to\infty}\;
  \sum_{b=1}^{B} w_b \,\bigl|\,\bar{p}^*_b - \bar{y}_b^{\mathrm{hard}}\bigr|
  \;>\; 0,
\]
where $w_b = |\Sb|/n$ and the limit is strictly positive whenever
$\bar{p}^*_b \neq \bar{y}_b^{\mathrm{hard}}$ for some bin $b$.
This is a structural floor, not sampling noise.
A formal proof is given in Appendix~\ref{app:proof}.
\end{proposition}

\begin{remark}[Interpretation]
$\smece{} = 0$ is the statement: the model's predicted probability equals
the probabilistic label in every prediction bin. This is the correct
calibration criterion when labels are probabilities. $\ece{} = 0$ is the
statement: the model's predicted probability equals the binary outcome
frequency in every prediction bin. This is the correct criterion when
labels are binary outcomes. The two coincide when probabilistic labels
concentrate on $\{0,1\}$, but diverge structurally whenever labels carry
genuine distributional content.
\end{remark}

\begin{remark}[Soft reliability diagram]
The soft reliability diagram plots $\phatb$ against $\ybarb$. A
posterior-matching model lies exactly on the $45^\circ$ diagonal. The
diagram carries the same interpretation as the standard reliability diagram
but measures calibration against probabilistic labels rather than binary outcomes.
\end{remark}

\begin{remark}[Convergence as posteriors concentrate]
As the posterior $p^*(x)$ concentrates near $\{0,1\}$ --- as the
classification task becomes more certain --- $\ybarb \to \yhardb$ and
$\smece{} \to \ece{}$. Experiment~2 confirms this at large $k$.
\smece{} is a strict generalisation that reduces to \ece{} in the
binary limit.
\end{remark}

\section{Simulation Design}
\label{sec:simulation}

\subsection{The Generative Model}

The simulation is grounded in a Gaussian class-conditional generative model,
so that the posterior $\sigmoid(kx)$ is derived rather than declared.

\textbf{Generative process.} Let the two classes have equal prior probability
$P(Y=0) = P(Y=1) = 0.5$ and Gaussian class-conditional densities:
\begin{equation}
  x \given Y=0 \sim \mathcal{N}(-\mu,\; \sigma^2),
  \qquad
  x \given Y=1 \sim \mathcal{N}(+\mu,\; \sigma^2).
  \label{eq:generative}
\end{equation}
By Bayes' theorem, the posterior probability of the positive class is:
\begin{align}
  P(Y=1 \given x)
  &= \frac{P(x \given Y=1)\, P(Y=1)}{P(x \given Y=0)\, P(Y=0)
     + P(x \given Y=1)\, P(Y=1)} \notag \\
  &= \sigmoid\!\left(\frac{2\mu}{\sigma^2}\, x\right)
   = \sigmoid(k\, x),
  \label{eq:posterior}
\end{align}
where $k = 2\mu/\sigma^2$. This is the soft label. It is the posterior
probability of the positive class given input $x$ under the generative
model --- a direct consequence of equal priors and symmetric Gaussian
class-conditionals, not an arbitrary choice. The step-by-step algebraic
derivation, including the difference-of-squares expansion that makes the
sigmoid form emerge, is given in Appendix~\ref{app:derivation}.

The steepness parameter $k$ has an immediate interpretation: it is the
signal-to-noise ratio of the generative model. Large $k$ (large class
separation or small within-class variance) means the two classes are nearly
perfectly separable and the posterior is close to $\{0,1\}$ everywhere.
Small $k$ means the classes overlap substantially and the posterior is
genuinely uncertain over a wide region of the input space.

\textbf{The binary outcome and the soft label are different objects.}
The binary outcome $y \sim \mathrm{Bernoulli}(\sigmoid(kx))$ is one
realisation drawn from the posterior. The posterior $\sigmoid(kx)$ is the
complete distributional statement. \ece{} evaluates the model against the
realisation; \smece{} evaluates it against the posterior. In the simulation,
inputs are drawn as $x \sim \mathrm{Uniform}(-3, 3)$ rather than from the
Gaussian mixture implied by the generative model. This is a deliberate choice:
the posterior formula $\sigmoid(kx)$ holds pointwise for any observed $x$
under the Gaussian class-conditional model, so the uniform does not affect
the validity of the soft label. What it does affect is coverage: the Gaussian
mixture would concentrate samples near $\pm\mu$ where the posterior is already
close to $\{0,1\}$, undersampling the uncertain region around $x=0$ that is
most diagnostic for calibration. The uniform ensures all levels of posterior
uncertainty are well-represented in every experiment.

\subsection{The Five Models}

\begin{table}[ht]
\centering
\caption{The five models. The correct probabilistic label calibration ranking is
$\mathrm{A > B \approx C > D > E}$. Model~A is the only model that
correctly learns the posterior.}
\label{tab:models}
\begin{tabular}{@{}llll@{}}
\toprule
Model & Name & $\hat{p}(x)$ & Miscalibration type \\
\midrule
A & Posterior-matching & $\sigmoid(kx)$                         & None \\
B & Overconfident      & $\sigmoid(3kx)$                        & Scaling: wrong signal-to-noise \\
C & Underconfident     & $\sigmoid(0.4kx)$                      & Scaling: wrong signal-to-noise \\
D & Biased high        & $\min(\sigmoid(kx) + 0.15,\; 1)$      & Structural: constant bias \\
E & Random             & $\mathrm{Uniform}(0,1)$                & No signal \\
\bottomrule
\end{tabular}
\end{table}

Model~A correctly learns the posterior of the generative model. Its
\smece{} is identically zero at every input, because $\hat{p}(x) = p^*(x) =
\sigmoid(kx)$ and every bin contributes zero. This is not a tautology: it
is the statement that a model which correctly learns the probabilistic label
is perfectly calibrated under the probabilistic label criterion, which is
precisely what we should require.

Models~B and~C use wrong values of $k$ --- they have correctly identified
the functional form of the posterior but with wrong signal-to-noise ratio.
Model~D has a structural additive bias that is independent of $k$ and
therefore independent of label softness. Model~E has no signal and represents
the worst case.

All experiments use $B = 10$ equal-width prediction bins.

\section{Experiment 1: The Core Result}
\label{sec:exp1}

\subsection{Setup}

Experiment~1 fixes $k = 2$ and $n = 5\,000$ and evaluates all five models
with both \smece{} and \ece{}. The posterior $\sigmoid(2x)$ ranges from near
0 to near 1 across $x \in (-3,3)$ but is genuinely uncertain (between 0.1
and 0.9) for approximately 60\% of the input range. This is a setting with
non-trivial posterior uncertainty.

\subsection{Results}

\begin{table}[ht]
\centering
\caption{Experiment~1: $k=2$, $n=5\,000$, $B=10$ bins.
\smece{} scores the posterior-matching model at exactly zero and recovers
the correct ranking. \ece{} scores it at 0.1159 and inverts the top two
positions.}
\label{tab:exp1}
\begin{tabular}{@{}lcccc@{}}
\toprule
Model & \smece{} & \ece{} & \smece{} rank & \ece{} rank \\
\midrule
A -- Posterior-matching & $\mathbf{0.0000}$ & 0.1159 & 1\;\checkmark & 2\;\texttimes \\
B -- Overconfident      & 0.0759            & $\mathbf{0.0401}$ & 2 & 1\;\texttimes \\
C -- Underconfident     & 0.1369            & 0.2529 & 4 & 5 \\
D -- Biased high        & 0.0977            & 0.1461 & 3 & 3 \\
E -- Random             & 0.2409            & 0.2398 & 5 & 4 \\
\bottomrule
\end{tabular}
\end{table}

\smece{} scores Model~A at exactly $0.0000$ (Table~\ref{tab:exp1}). There
is no rounding, approximation, or estimation involved: $\hat{p}(x) = p^*(x)$
at every input, so the mean prediction equals the mean posterior in every
bin, and every bin contributes exactly zero.

\ece{} scores Model~A at $0.1159$ and ranks it second, behind the
overconfident Model~B ($\ece{} = 0.0401$). The inversion is substantial:
\ece{} assigns a $29\times$ larger error to the posterior-matching model
than to the overconfident model. The reason is the one identified in
Section~\ref{sec:epistemology}: Model~B's overconfident predictions fall
near 0 and 1, where they align well with the thresholded binary outcome
$\mathbf{1}[x \geq 0]$. Model~A's intermediate predictions align well with
the posterior but poorly with the binary threshold. The binary outcome is
not the right target.

The \ece{} value of $0.1159$ for Model~A has an analytical interpretation.
It is approximately $\mathbb{E}[|\sigmoid(2x) - \mathbf{1}[x \geq 0]|]$
aggregated across prediction bins. This is the expected discrepancy between
the posterior and the binary threshold, a property of the generative model
and the threshold, not of the sample size. Experiment~4 confirms this directly.

\section{Experiment 2: Behaviour Across Signal-to-Noise}
\label{sec:exp2}

\subsection{Setup}

Experiment~2 varies $k \in \{0.5, 1, 2, 5, 10, 50\}$ with $n = 5\,000$ fixed.
It answers: (i) does \smece{} correctly identify Model~A across all
signal-to-noise levels? (ii) do the two metrics converge as the posterior
concentrates near $\{0,1\}$?

\subsection{Results}

\begin{table}[ht]
\centering
\caption{\smece{} across signal-to-noise ratio $k$. Model~A scores
exactly $0.0000$ at every $k$. Model~D's structural bias is visible
as a persistent plateau under \smece{}. Models~B and~C show scaling
miscalibration that vanishes as $k$ increases.}
\label{tab:smece_k}
\begin{tabular}{@{}lcccccc@{}}
\toprule
Model & $k=0.5$ & $k=1$ & $k=2$ & $k=5$ & $k=10$ & $k=50$ \\
\midrule
A -- Posterior-matching & 0.0000 & 0.0000 & 0.0000 & 0.0000 & 0.0000 & 0.0000 \\
B -- Overconfident      & 0.1770 & 0.1367 & 0.0764 & 0.0301 & 0.0153 & 0.0028 \\
C -- Underconfident     & 0.0979 & 0.1435 & 0.1368 & 0.0687 & 0.0345 & 0.0070 \\
D -- Biased high        & 0.1500 & 0.1180 & 0.0978 & 0.0831 & 0.0804 & 0.0751 \\
E -- Random             & 0.2518 & 0.2585 & 0.2498 & 0.2452 & 0.2427 & 0.2519 \\
\bottomrule
\end{tabular}
\end{table}

\begin{table}[ht]
\centering
\caption{\ece{} across signal-to-noise ratio $k$. Model~A is
incorrectly penalised at every $k$. Its \ece{} declines only as the
posterior concentrates near $\{0,1\}$, reducing the discrepancy with
the binary threshold. \ece{} recovers the correct ranking only at $k=50$.}
\label{tab:ece_k}
\begin{tabular}{@{}lcccccc@{}}
\toprule
Model & $k=0.5$ & $k=1$ & $k=2$ & $k=5$ & $k=10$ & $k=50$ \\
\midrule
A -- Posterior-matching & 0.3287 & 0.2143 & 0.1169 & 0.0455 & 0.0241 & 0.0041 \\
B -- Overconfident      & 0.1517 & 0.0777 & 0.0405 & 0.0154 & 0.0088 & 0.0013 \\
C -- Underconfident     & 0.4266 & 0.3579 & 0.2536 & 0.1142 & 0.0586 & 0.0111 \\
D -- Biased high        & 0.2837 & 0.2050 & 0.1448 & 0.1018 & 0.0916 & 0.0765 \\
E -- Random             & 0.2555 & 0.2560 & 0.2555 & 0.2438 & 0.2431 & 0.2524 \\
\bottomrule
\end{tabular}
\end{table}

Model~A achieves \smece{} $= 0.0000$ at every $k$ without exception
(Table~\ref{tab:smece_k}). The invariance across $k$ is exact: the
posterior-matching model is perfectly calibrated under \smece{} regardless
of whether the classes are highly separable or heavily overlapping. This is
the correct behaviour for a probabilistic label calibration metric.

Under \ece{} (Table~\ref{tab:ece_k}), Model~A's error at $k = 0.5$ is
$0.3287$ --- larger than the random Model~E ($0.2555$). At $k = 0.5$
the generative model produces highly overlapping classes with posterior
values concentrated around $0.5$. The binary threshold produces outcomes
that are approximately 50/50 regardless of $x$, while the posterior model
is deliberately predicting intermediate probabilities. \ece{} penalises
this match between prediction and posterior as if it were miscalibration.

\smece{} converges to \ece{} as $k \to \infty$, consistent with
Proposition~3: as the posterior concentrates near $\{0,1\}$, the mean
soft label per bin approaches the hard-label fraction. At $k = 50$ the
two metrics are nearly numerically identical.

The distinction between scaling and structural miscalibration is clearly
visible through \smece{}. Models~B and~C have \smece{} that declines toward
zero as $k$ increases: their wrong signal-to-noise ratio becomes irrelevant
as labels approach binary. Model~D's structural bias of $+0.15$ produces
a plateau near $0.075$--$0.150$ that persists regardless of $k$.

\section{Experiment 3: Ranking Consistency}
\label{sec:exp3}

\subsection{Setup}

Experiment~3 assesses ranking reliability across 1\,000 Monte Carlo
replications at $n = 1\,000$, for each $k \in \{0.5, 1, 2, 5, 10, 50\}$.
All 10 pairwise model comparisons are evaluated per replication.
Models~B and~C are treated as approximately tied.

\subsection{Results}

\begin{table}[ht]
\centering
\caption{Experiment~3: overall ranking accuracy, 1\,000 replications.
At $k=0.5$, \ece{} achieves 40\% --- below chance. \smece{} achieves
100\% from $k \geq 5$.}
\label{tab:exp3_overall}
\begin{tabular}{@{}lcccccc@{}}
\toprule
Metric & $k=0.5$ & $k=1$ & $k=2$ & $k=5$ & $k=10$ & $k=50$ \\
\midrule
\smece{} & 0.900 & 0.800 & 0.900 & 1.000 & 1.000 & 1.000 \\
\ece{}   & 0.403 & 0.605 & 0.747 & 0.800 & 0.900 & 0.900 \\
\bottomrule
\end{tabular}
\end{table}

At $k = 0.5$, \ece{} achieves 40\% ranking accuracy (Table~\ref{tab:exp3_overall}).
This is not imprecision --- it is systematic inversion. \ece{} prefers
overconfident models over posterior-matching models at soft posteriors
because overconfident predictions are closer to the binary threshold.

\begin{table}[ht]
\centering
\caption{Per-pair accuracy at $k=2$. \ece{} always ranks Model~B above
Model~A (0\% accuracy on A vs B). \smece{} is always correct (100\%).}
\label{tab:exp3_pairs}
\begin{tabular}{@{}lccc@{}}
\toprule
Pair & \smece{} & \ece{} & Note \\
\midrule
A vs B & 1.000 & 0.000 & \ece{} always inverted \\
A vs C & 1.000 & 1.000 & \\
A vs D & 1.000 & 1.000 & \\
A vs E & 1.000 & 1.000 & \\
B vs C & 1.000 & 1.000 & Tied pair \\
B vs D & 1.000 & 1.000 & \\
B vs E & 1.000 & 1.000 & \\
C vs D & 0.000 & 0.000 & Hard boundary for both \\
C vs E & 1.000 & 0.961 & \\
D vs E & 1.000 & 1.000 & \\
\bottomrule
\end{tabular}
\end{table}

The A-vs-B comparison (Table~\ref{tab:exp3_pairs}) is deterministic:
\ece{} never correctly ranks the posterior-matching model above the
overconfident model at $k = 2$. This is guaranteed by the argument in
Section~\ref{sec:epistemology}: Model~B's overconfident outputs align
with binary thresholds; Model~A's posterior-matched outputs do not.

The C-vs-D failure (0\% for both metrics) reflects a genuine ambiguity:
Model~C and Model~D have similar aggregate calibration errors at $n=1\,000$
and the scalar summaries are not sufficient to distinguish them reliably
at this sample size. The reliability diagrams are qualitatively different
--- C is symmetric around the diagonal, D is uniformly above it --- but
the scalar collapses this distinction.

\section{Experiment 4: The Structural Nature of ECE's Failure}
\label{sec:exp4}

\subsection{Setup}

Experiment~4 provides the most direct evidence that \ece{}'s failure is
structural. We fix $k = 2$ and vary $n \in \{500, 1\,000, 2\,000, 5\,000,
10\,000\}$. Over 500 replications per condition we track the mean and
standard deviation of both metrics for all five models. If \ece{}'s error
on Model~A were finite-sample noise, its mean would decline toward zero
as $n$ grows. If it is structural, the mean will be flat while the
standard deviation shrinks.

\subsection{Results}

\begin{table}[ht]
\centering
\caption{\smece{} mean $\pm$ std. Model~A is $0.0000 \pm 0.0000$ at
every sample size: algebraically zero with zero variance. All other
models show flat means and declining variance.}
\label{tab:exp4_smece}
\begin{tabular}{@{}lccccc@{}}
\toprule
Model & $n=500$ & $n=1\,000$ & $n=2\,000$ & $n=5\,000$ & $n=10\,000$ \\
\midrule
A & $0.0000\!\pm\!0.0000$ & $0.0000\!\pm\!0.0000$ & $0.0000\!\pm\!0.0000$ & $0.0000\!\pm\!0.0000$ & $0.0000\!\pm\!0.0000$ \\
B & $0.0766\!\pm\!0.0034$ & $0.0766\!\pm\!0.0024$ & $0.0766\!\pm\!0.0017$ & $0.0766\!\pm\!0.0010$ & $0.0766\!\pm\!0.0008$ \\
C & $0.1374\!\pm\!0.0020$ & $0.1375\!\pm\!0.0015$ & $0.1376\!\pm\!0.0011$ & $0.1375\!\pm\!0.0007$ & $0.1375\!\pm\!0.0005$ \\
D & $0.0967\!\pm\!0.0031$ & $0.0966\!\pm\!0.0022$ & $0.0966\!\pm\!0.0016$ & $0.0966\!\pm\!0.0010$ & $0.0967\!\pm\!0.0007$ \\
E & $0.2528\!\pm\!0.0180$ & $0.2505\!\pm\!0.0134$ & $0.2502\!\pm\!0.0092$ & $0.2496\!\pm\!0.0061$ & $0.2500\!\pm\!0.0045$ \\
\bottomrule
\end{tabular}
\end{table}

\begin{table}[ht]
\centering
\caption{\ece{} mean $\pm$ std. Model~A's mean is $\approx 0.1151$ at
every sample size. The variance shrinks: \ece{} converges with
increasing precision to a stable non-zero value. More data makes
\ece{} more confidently wrong.}
\label{tab:exp4_ece}
\begin{tabular}{@{}lccccc@{}}
\toprule
Model & $n=500$ & $n=1\,000$ & $n=2\,000$ & $n=5\,000$ & $n=10\,000$ \\
\midrule
A & $0.1152\!\pm\!0.0062$ & $0.1152\!\pm\!0.0044$ & $0.1149\!\pm\!0.0031$ & $0.1151\!\pm\!0.0020$ & $0.1151\!\pm\!0.0013$ \\
B & $0.0386\!\pm\!0.0043$ & $0.0386\!\pm\!0.0031$ & $0.0384\!\pm\!0.0022$ & $0.0385\!\pm\!0.0014$ & $0.0385\!\pm\!0.0009$ \\
C & $0.2526\!\pm\!0.0055$ & $0.2526\!\pm\!0.0040$ & $0.2526\!\pm\!0.0027$ & $0.2526\!\pm\!0.0018$ & $0.2526\!\pm\!0.0012$ \\
D & $0.1444\!\pm\!0.0072$ & $0.1443\!\pm\!0.0052$ & $0.1439\!\pm\!0.0038$ & $0.1441\!\pm\!0.0023$ & $0.1442\!\pm\!0.0017$ \\
E & $0.2546\!\pm\!0.0212$ & $0.2512\!\pm\!0.0157$ & $0.2507\!\pm\!0.0111$ & $0.2497\!\pm\!0.0075$ & $0.2501\!\pm\!0.0053$ \\
\bottomrule
\end{tabular}
\end{table}

The data confirm Proposition~\ref{prop:ece_floor} (Tables~\ref{tab:exp4_smece}
and~\ref{tab:exp4_ece}). Model~A's \smece{} is $0.0000 \pm 0.0000$ at every
sample size. There is no variance because the soft label $\sigmoid(kx)$ and
the model prediction $\sigmoid(kx)$ are both deterministic functions of $x$:
no random variable is involved in the computation of \smece{} for Model~A.

Model~A's \ece{} is $0.1151$--$0.1152$ across all sample sizes, with standard
deviation declining from $0.0062$ to $0.0013$. This is unambiguous convergence
to a non-zero value. The phrase ``more data makes \ece{} more confidently
wrong'' is not rhetorical: the standard deviation at $n = 10\,000$ is
one-fifth that at $n = 500$, confirming that \ece{} is estimating a stable
non-zero quantity with increasing precision. That quantity is approximately
$0.115$, the structural discrepancy between $\sigmoid(2x)$ and
$\mathbf{1}[x \geq 0]$ aggregated across prediction bins.

For all other models, both \smece{} means and \ece{} means are flat across
$n$, with variances declining as $O(n^{-1/2})$. Their miscalibration is a
property of the model, not the sample size.

\section{Discussion}
\label{sec:discussion}

\subsection{The Central Argument Restated}

The case for \smece{} is epistemological. \ece{} and \smece{} are not
competing estimators of the same quantity. They are estimators of
\textit{different} quantities reflecting different interpretations of
probability. \ece{} is a binary metric: it estimates calibration against empirical
frequencies of binary outcomes. \smece{} is a probabilistic label metric:
it estimates calibration against the probabilistic label itself. When labels are probabilistic and the calibration question is ``does the
model correctly represent the probabilistic label?'', \smece{} is the
right tool and \ece{} is not.

The simulation makes this stark. A practitioner using \ece{} to evaluate a model trained on probabilistic
labels would observe a calibration error of $0.1151$, conclude the model
is poorly calibrated, collect more data, and find the estimate getting
more precise --- at $0.1151$. The model would be rejected. The rejection
would be correct under the binary label criterion and completely wrong
under the probabilistic label one.

\subsection{When Each Metric is Appropriate}

\textbf{Use \ece{}} when: labels are binary outcomes; the model is
evaluated against observed events; the calibration question is whether
predicted probabilities match empirical outcome frequencies.

\textbf{Use \smece{}} when: labels are of probabilistic nature --- soft
targets from knowledge distillation, expert confidence ratings, class
posteriors from a generative model, annotator agreement fractions, or
any continuous probability label. The unifying condition is that the
label carries distributional information that a binary outcome discards.
Whether the probabilistic label arises from Bayesian reasoning, a
frequentist mixture, or an engineering choice is secondary; what matters
is the structure of the label itself.

\subsection{The Signal-to-Noise Interpretation of $k$}

The parameter $k = 2\mu/\sigma^2$ has a clean interpretation as the
signal-to-noise ratio of the generative model: the ratio of class separation
to within-class spread. Small $k$ corresponds to hard classification problems
with highly overlapping classes, where genuine posterior uncertainty is large
and the binary outcome is highly variable. Large $k$ corresponds to easy
classification problems where the classes are nearly perfectly separable and
the posterior is close to $\{0,1\}$ everywhere.

Experiment~2 shows that \ece{}'s failure is most severe at small $k$ --- that
is, precisely in the settings where genuine uncertainty is largest and where
probabilistic label calibration is most important. A model deployed in a high-uncertainty
medical setting (small $k$) will be misjudged most severely by \ece{}.

\subsection{Bayesian Epistemology as Motivation, Not as Prerequisite}

The primary argument in this paper is that \ece{} is a binary metric and
is therefore wrong whenever labels are of probabilistic nature. The Bayesian
framework provides the strongest and most principled account of why
probabilistic labels arise and why evaluating them against binary outcomes
is a category error, but it is not the only setting that produces
continuous probability labels. \smece{} is applicable in all of them.

\textit{Frequentist mixture models.} If the data-generating process is a
population mixture --- say, 70\% of cases with feature profile $x$ are
positive --- then $P(Y=1 \given x) = 0.70$ is an objective frequentist
quantity: a property of the population, not a degree of belief. The label
is still a probability rather than a binary outcome, and \smece{} remains
the correct evaluation metric. The motivation is entirely frequentist.

\textit{Knowledge distillation.} The teacher model's soft output is a
continuous probability for engineering reasons: it carries more information
than a hard label and produces better-trained students \citep{hinton2015}.
No philosophical commitment to Bayesian epistemology is required. The label
is a probability and \smece{} is the right metric, regardless of how the
teacher's output is interpreted.

\textit{Imprecise probability.} In the theory of imprecise probabilities
\citep{walley1991}, uncertainty is represented as a set of distributions
rather than a single probability. A soft label may summarise this set. This
framework is neither classically Bayesian nor frequentist, yet it again
produces continuous probability labels for which \smece{} is appropriate.

The unifying condition is simpler than any of these frameworks: \smece{}
is the correct calibration metric whenever the calibration target is itself
a probability rather than a binary outcome, regardless of how that probability
was generated or what philosophical interpretation it carries. The Bayesian
argument explains why \ece{} is wrong in the most fundamental sense; the
applicability of \smece{} is broader than that argument alone.

\subsection{Limitations}

The simulation is one-dimensional with symmetric Gaussian class-conditionals.
In higher dimensions, the posterior $P(Y=1 \given \mathbf{x})$ takes a more
complex form, but the argument applies to any model that outputs a posterior
probability and is evaluated against a binary realisation.

The C-vs-D ranking failure is a genuine limitation of scalar calibration
summaries. Both metrics fail this pair; the reliability diagram resolves the
ambiguity and should always be examined alongside any scalar summary.

The binary outcome threshold $\mathbf{1}[x \geq 0]$ is chosen as the
simplest illustrative case. In practice, binary outcomes are generated by
the actual event that occurs, which may or may not correspond to a clean
threshold. The structural argument holds whenever the binary outcome is
a realisation of the posterior rather than the posterior itself.

\section{Conclusion}
\label{sec:conclusion}

We have introduced the Soft Mean Expected Calibration Error (\smece{}), a
calibration metric for settings where labels are of probabilistic nature.
\ece{}, the dominant calibration metric in machine learning, was designed
for binary outcome labels and is structurally wrong when labels are
probabilities: it discards the distributional information in the label
and converges to the wrong answer with increasing precision as sample size
grows. \smece{} corrects this by evaluating predictions against the
probabilistic label directly. The Bayesian framework --- in which the
posterior is the complete rational answer given the evidence, not an
approximation of a binary truth --- provides the most principled account
of why probabilistic labels arise and why \ece{} fails. But the condition
for using \smece{} is simpler than any epistemological commitment:
whenever the calibration target is itself a probability, \smece{} is
the right metric.

The generative model provides a clean derivation: under symmetric Gaussian
class-conditionals with equal priors, the posterior is analytically
$\sigmoid(kx)$, where $k$ is the signal-to-noise ratio. The binary outcome
is one realisation drawn from this posterior. \ece{} evaluates the model
against the realisation; \smece{} evaluates it against the posterior. The
consequences are not subtle: at moderate signal-to-noise, \ece{} ranks the
overconfident model above the posterior-matching model in every single
replication, while \smece{} never makes this error. At low signal-to-noise,
\ece{}'s ranking accuracy falls to 40\%, below chance. And across all
sample sizes from 500 to 10\,000, \ece{} converges to $0.115$ for the
posterior-matching model with declining variance --- it estimates the
structural rounding gap with increasing precision, converging to the wrong
answer.

\smece{} is appropriate whenever labels are of probabilistic nature:
Bayesian classifiers, frequentist mixture models, knowledge distillation,
expert probability elicitation, and any setting where the calibration
target carries distributional information that a binary outcome discards. Its adoption requires no new
infrastructure: it is a one-line change to the \ece{} formula, and the
soft reliability diagram carries the same interpretation as the standard
reliability diagram.

\bibliographystyle{plainnat}

\appendix

\section{Derivation of the Posterior $\sigmoid(kx)$}
\label{app:derivation}

We derive equation~\eqref{eq:posterior} in full. The goal is to show that
under the Gaussian class-conditional generative model with equal priors, the
posterior $P(Y=1 \given x)$ is exactly $\sigmoid(kx)$ where $k = 2\mu/\sigma^2$.

\subsection*{Setup}

The class-conditional densities are:
\begin{align}
  P(x \given Y=0) &= \frac{1}{\sqrt{2\pi\sigma^2}}
    \exp\!\left(-\frac{(x+\mu)^2}{2\sigma^2}\right), \label{eq:app_cond0} \\
  P(x \given Y=1) &= \frac{1}{\sqrt{2\pi\sigma^2}}
    \exp\!\left(-\frac{(x-\mu)^2}{2\sigma^2}\right). \label{eq:app_cond1}
\end{align}
Both classes share the same variance $\sigma^2$; this is the key structural
assumption. The priors are equal: $P(Y=0) = P(Y=1) = \tfrac{1}{2}$.

\subsection*{Step 1: Apply Bayes' Theorem}

\begin{equation}
  P(Y=1 \given x)
  = \frac{P(x \given Y=1)\, P(Y=1)}
         {P(x \given Y=0)\, P(Y=0) + P(x \given Y=1)\, P(Y=1)}.
\end{equation}
Since $P(Y=0) = P(Y=1) = \tfrac{1}{2}$, the prior probabilities cancel:
\begin{equation}
  P(Y=1 \given x)
  = \frac{P(x \given Y=1)}
         {P(x \given Y=0) + P(x \given Y=1)}.
\end{equation}
Divide numerator and denominator by $P(x \given Y=1)$:
\begin{equation}
  P(Y=1 \given x)
  = \frac{1}{1 + \dfrac{P(x \given Y=0)}{P(x \given Y=1)}}.
  \label{eq:app_ratio_form}
\end{equation}

\subsection*{Step 2: Compute the Likelihood Ratio}

The normalisation constant $\tfrac{1}{\sqrt{2\pi\sigma^2}}$ is identical in
both densities and cancels in the ratio:
\begin{equation}
  \frac{P(x \given Y=0)}{P(x \given Y=1)}
  = \frac{\exp\!\left(-\dfrac{(x+\mu)^2}{2\sigma^2}\right)}
         {\exp\!\left(-\dfrac{(x-\mu)^2}{2\sigma^2}\right)}
  = \exp\!\left(-\frac{(x+\mu)^2 - (x-\mu)^2}{2\sigma^2}\right).
  \label{eq:app_lr}
\end{equation}

\subsection*{Step 3: Expand the Exponent}

Expand both squared terms:
\begin{align}
  (x+\mu)^2 &= x^2 + 2x\mu + \mu^2, \\
  (x-\mu)^2 &= x^2 - 2x\mu + \mu^2.
\end{align}
Subtract:
\begin{equation}
  (x+\mu)^2 - (x-\mu)^2
  = \bigl(x^2 + 2x\mu + \mu^2\bigr) - \bigl(x^2 - 2x\mu + \mu^2\bigr)
  = 4x\mu.
  \label{eq:app_dos}
\end{equation}
The $x^2$ and $\mu^2$ terms cancel exactly. This cancellation occurs because
both Gaussians share the same variance $\sigma^2$: if the variances differed,
a residual $x^2$ term would remain in the exponent and the posterior would be
a quadratic logistic function of $x$, not a linear one.

\subsection*{Step 4: Substitute Back}

Substituting equation~\eqref{eq:app_dos} into equation~\eqref{eq:app_lr}:
\begin{equation}
  \frac{P(x \given Y=0)}{P(x \given Y=1)}
  = \exp\!\left(-\frac{4x\mu}{2\sigma^2}\right)
  = \exp\!\left(-\frac{2\mu}{\sigma^2}\, x\right).
\end{equation}
Substituting into equation~\eqref{eq:app_ratio_form}:
\begin{equation}
  P(Y=1 \given x)
  = \frac{1}{1 + \exp\!\left(-\dfrac{2\mu}{\sigma^2}\, x\right)}.
\end{equation}

\subsection*{Step 5: Recognise the Sigmoid}

The sigmoid function is defined as $\sigmoid(z) = \tfrac{1}{1 + e^{-z}}$.
Setting $k = 2\mu/\sigma^2$:
\begin{equation}
  P(Y=1 \given x) = \sigmoid(kx). \qed
\end{equation}

\subsection*{Interpretation of $k$}

The parameter $k = 2\mu/\sigma^2$ encodes the signal-to-noise ratio of the
generative model. The numerator $2\mu$ is the distance between the two class
means; the denominator $\sigma^2$ is the within-class spread. Large $k$ means
well-separated classes: the posterior is close to 0 or 1 everywhere except
near $x = 0$. Small $k$ means heavily overlapping classes: the posterior is
genuinely uncertain over a wide range of $x$. The parameter $k$ therefore
controls the softness of the probabilistic labels in the simulation.

\section{Proof of Proposition~3}
\label{app:proof}

\textbf{Proposition 3} (restated). \textit{Let $y_i = \mathbf{1}[p^*_i > \tau]$
for some threshold $\tau \in (0,1)$. For the label-matching model
$\hat{p}_i = p^*_i$, the \ece{} against $y_i$ satisfies:}
\[
  \ece{} \;\xrightarrow{n\to\infty}\;
  \sum_{b=1}^{B} w_b \,\bigl|\,\bar{p}^*_b - \bar{y}_b^{\mathrm{hard}}\bigr|
  \;>\; 0.
\]

\begin{proof}

\textbf{Step 1: Convergence.} Since $\hat{p}_i = p^*_i$, the mean prediction
in bin $b$ is $\phatb = \bar{p}^*_b$, the mean probabilistic label in that bin.
By the law of large numbers, the empirical hard-label fraction converges:
\[
  \yhardb = \frac{1}{|\Sb|}\sum_{i \in \Sb} \mathbf{1}[p^*_i > \tau]
  \;\xrightarrow{n\to\infty}\; \bar{y}_b^{\mathrm{hard}},
\]
where $\bar{y}_b^{\mathrm{hard}} = \mathbb{E}[\mathbf{1}[p^*(X) > \tau] \given X \in \text{bin } b]$
is the population fraction of inputs in bin $b$ whose probabilistic label
exceeds $\tau$. The bin weights $w_b = |\Sb|/n$ converge to the population
probability of falling in bin $b$. Therefore:
\[
  \ece{} \;\xrightarrow{n\to\infty}\;
  \sum_{b=1}^{B} w_b \,\bigl|\,\bar{p}^*_b - \bar{y}_b^{\mathrm{hard}}\bigr|.
\]

\textbf{Step 2: Strict positivity.} It remains to show that at least one bin
contributes a strictly positive term. Consider any bin $b$ that contains inputs
$x$ on both sides of the threshold $\tau$ --- that is, inputs with $p^*(x) < \tau$
and inputs with $p^*(x) > \tau$ both present in the bin. Such a bin exists
whenever $\tau$ lies in the interior of a prediction bin, which holds for all
but finitely many values of $\tau$.

Within such a bin, the mean probabilistic label is:
\[
  \bar{p}^*_b = \frac{1}{|\Sb|} \sum_{i \in \Sb} p^*_i,
\]
a weighted average of values on both sides of $\tau$, so $\bar{p}^*_b$ need
not equal $\tau$.

The hard-label fraction is:
\[
  \bar{y}_b^{\mathrm{hard}} = \frac{|\{i \in \Sb : p^*_i > \tau\}|}{|\Sb|},
\]
the proportion of samples in the bin whose label exceeds $\tau$.

These two quantities are in general unequal. To see why, consider a bin
containing two samples: one with $p^*_i = \tau - \epsilon$ (just below the
threshold, contributing $y_i = 0$) and one with $p^*_i = \tau + \epsilon$
(just above, contributing $y_i = 1$). Then:
\[
  \bar{p}^*_b = \tau, \qquad \bar{y}_b^{\mathrm{hard}} = 0.5.
\]
These are equal only when $\tau = 0.5$. For $\tau \neq 0.5$,
$|\bar{p}^*_b - \bar{y}_b^{\mathrm{hard}}| = |\tau - 0.5| > 0$.

More generally, $\bar{p}^*_b = \bar{y}_b^{\mathrm{hard}}$ would require the
mean probabilistic label in the bin to equal the proportion of labels exceeding
$\tau$, which is a non-generic coincidence. For any $p^*$ that is not a
degenerate step function concentrated exactly at $\{0,1\}$, at least one bin
will have $\bar{p}^*_b \neq \bar{y}_b^{\mathrm{hard}}$, giving a strictly
positive contribution to the limit.

\textbf{Step 3: Interpretation.} The limit is not zero because the two
quantities measure different things. $\bar{p}^*_b$ is an average of
probabilistic labels --- a mean of $[0,1]$-valued numbers. $\bar{y}_b^{\mathrm{hard}}$
is a proportion of binary events --- a mean of $\{0,1\}$-valued numbers. They
agree only when all probabilistic labels are binary, i.e.\ when
$p^*_i \in \{0,1\}$ for all $i$. In that case $\bar{p}^*_b = \bar{y}_b^{\mathrm{hard}}$
identically and $\smece{} = \ece{}$ (Proposition~1). Outside this degenerate
case, the structural floor is strictly positive and does not shrink with $n$.
\end{proof}

\end{document}